\def\year{2020}\relax
\documentclass[letterpaper]{article}
\usepackage{aaai}
\usepackage{times}
\usepackage{helvet}
\usepackage{courier}
\usepackage{amsfonts}
\usepackage{amsmath}
\usepackage{amssymb}
\usepackage{amsthm}
\usepackage{graphicx}
\usepackage[dvipsnames]{xcolor}
\usepackage{balance}
\usepackage{tabularx}
\usepackage{multirow}
\usepackage{booktabs}
\usepackage{listings}
\usepackage[ruled,linesnumbered]{algorithm2e}

\theoremstyle{definition}
\newtheorem{definition}{Definition}
\newtheorem{coro}{Corollary}
\newtheorem{prop}{Proposition}

\newenvironment{sketch}{%
  \proof}{\endproof}
  
\newcommand{\pre}{pre}
\newcommand{\pe}{eff^+}
\newcommand{\de}{eff^-}
\newcommand{\poss}{\diamond}
\newcommand{\ppre}{\poss\pre}
\newcommand{\pde}{\poss\de}
\newcommand{\ppe}{\poss\pe}
\newcommand{\con}{\mathcal{M}_{con}}
\newcommand{\rel}{\mathcal{M}_{rel}}
\newcommand{\joe}{\mathcal{M}_{join}}
\newcommand{\hum}{\mathcal{M}_H}

\begin{document}
%
\title{Model Elicitation through Direct Questioning}
\author{
Sachin Grover$^+$ \and David Smith \and Subbarao Kambhampati$^+$ \\[0.5ex]
$^+$Arizona State University, Tempe, AZ 85281 USA\\[0.5ex]
{\tt \{ sgrover6, rao  \} @ asu.edu, david.smith@psresearch.xyz }
}
\date{}
\nocopyright
\maketitle

\begin{abstract}
The future will be replete with scenarios where humans are robots will be working together in complex environments. Teammates interact, and the robot's interaction has to be about getting useful information about the human's (teammate's) model. There are many challenges before a robot can interact, such as incorporating the structural differences in the human's model, ensuring simpler responses, etc. In this paper, we investigate how a robot can interact to localize the human model from a set of models. We show how to generate questions to refine the robot's understanding of the teammate's model. We evaluate the method in various planning domains. The evaluation shows that these questions can be generated offline, and can help refine the model through simple answers.
\end{abstract}

Humans work in complex environments by choosing different actions to reach their goals. For human-aware AI to be useful, the automated agents need to support humans in these complex environments.  
Recently, there has been a lot of work to explain the decisions taken by the agents \cite{lime-ml,explanations-mrp,miller2019explanation}. However, there is a need for two-way interaction to work with human teammates.
Even though a teammate might not explain their actions, but a robotic agent might need to ask directed questions to understand different parts of the human model.
These questions can refine the generalized models acquired by working in the environment by using humans behavioral data such as plan traces \cite{gil1994learning,stern2017efficient,zhuo2020discovering,garridolearning}.
Thus, in this paper, we look at how a robotic agent interacts with a teammate through directed questions, and simple answers can refine the model.

There are many challenges for asking directed questions, such as questions should be easy enough to understand, or it shouldn't overwhelm the teammate. Furthermore, the lack of understanding about the way humans think or handle information makes it difficult to ask questions. The {\em education} community has spent decades to understand how students conceptualize or interpret knowledge \cite{ortony1977representation},
such as, knowledge can be in the form of concept maps \cite{novak2006theory} or in the form of causal rules \cite{shultz1982rules}, or the form of {\em concept images} \cite{vinner1980concept}. There has also been some work on how we use these knowledge structures by understanding how students construct formal mathematical proofs \cite{moore1994making}. There have been studies to test various hypotheses, but the research is inconclusive in general conditions.

\begin{figure}[tbp]
    \centering
    \includegraphics[width=\columnwidth]{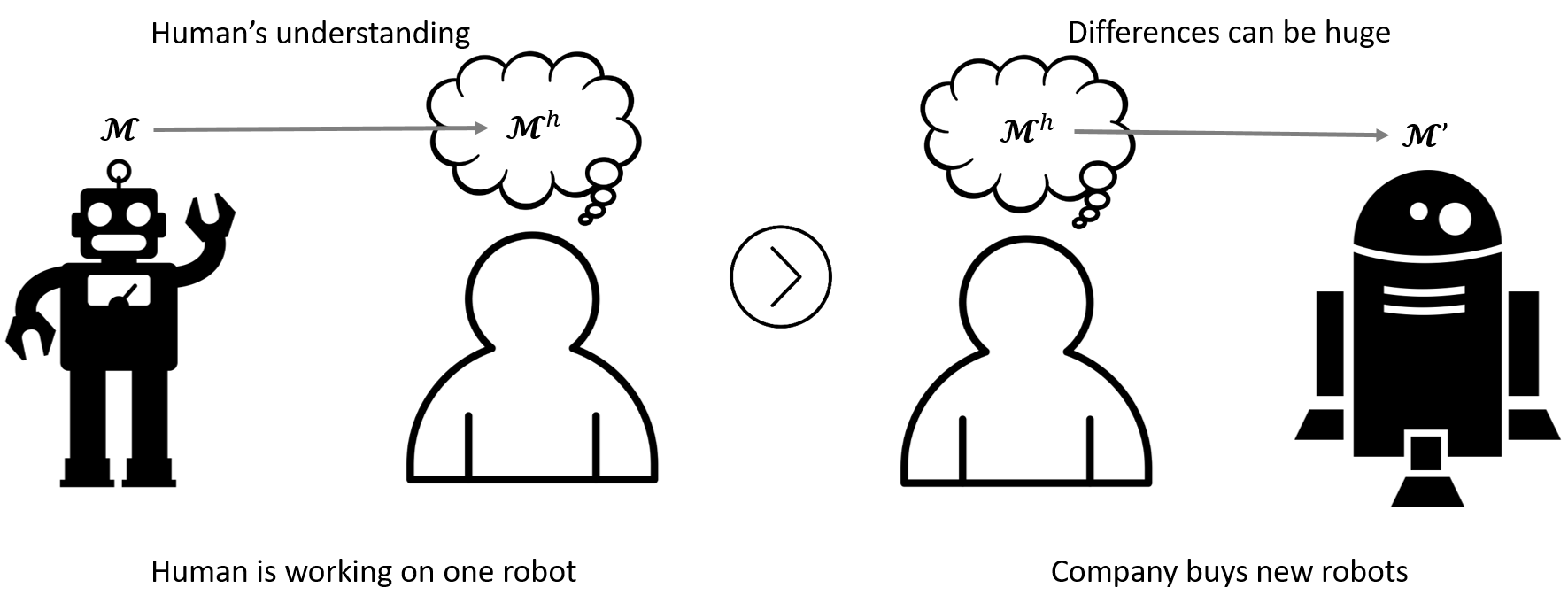}
    \caption{A human teammate is working on one of the robots and has an understanding $\hum$ of robot's model ($\mathcal{M}$). When the company buys a new robot (model $\mathcal{M}'$) but the human teammate's model $\hum$ does not change and can cause damage to the robot or effect the collaboration of the team.}
    \label{fig:scenario}
\end{figure}

In general, a question consists of two parts -- (1) the content of the query, and (2) the set of possible answers. One needs to ask a question where the answer can provide relevant information. However, an automated agent has structured knowledge with the closed-world assumption to calculate and evaluate their decisions and lacks any understanding of the human's mental model. 
Thus, one of the major challenges faced by an agent is to find a way to question the behavior of the teammate and understand various responses by the teammate to update their structured model.


{\em Human-in-the-loop} planning \cite{kambhampati2015human}, suggests to incorporate the teammate's model in the robot's planning process. This work supports the idea of hilp, where the agent can refine the robot's understanding of the teammate's model. These questions should be easy to understand and directed to refine the model. Thus, the major contributions for this work are --
\begin{itemize}
    \item[(C1)] generate questions while accounting for the gap between the structured model and the unstructured (or unknown structured) model of the teammate,
    \item[(C2)] incorporate the effect of uncertainty for generating questions,
    \item[(C3)] derive relevant information from simpler answers like -- (1) yes/no answers, or (2) the complete plan if possible.
\end{itemize}

\section*{Background}
In the intelligent tutoring system community, there has been a lot of work to represent the knowledge of a student, usually based on the number of problems they can solve. 
Sometimes, it can be represented using directed graphs, to depict a dependency among different knowledge concepts called {\em knowledge spaces} \cite{falmagne2006assessment,doignon2015knowledge}. 
We also use such dependency graphs to represent a robot's knowledge about their teammate's model using STRIPS \cite{fikes1971strips} as the causal dependency can be represented using cause and effect for an action.
Robust planning model \cite{nguyen2017robust} is used to represent the uncertainty, as some conditions that may or may not be a part of the human's model of the environment.
These probable conditions for preconditions and effects are called annotations. We will now formally describe the annotated model.


The correct human's model is represented as $\mathcal{M}_H$ is one of the concrete models in the set of models  $\mathbb{M} = \{\mathcal{M}_0, \mathcal{M}_1, \mathcal{M}_2, ..., \mathcal{M}_n\}$.
$\mathbb{M}$ is represented using a super-set of predicates and operators of each models as $\mathbb{M} = \langle\mathcal{F}, \mathcal{A}\rangle$.
$\mathcal{F}$ are the set of propositional state variables or {\em facts} and any state $s \in \mathcal{F}$.
Any action  $a \in \mathcal{A}$ is defined as $a = \langle pre(a), eff^{\pm}(a), \diamond pre(a), \diamond eff^{\pm}(a)\rangle$, where $pre(a), eff^{\pm}(a)$ are certain preconditions and effects and $\diamond pre(a), \diamond eff^{\pm}(a) \subseteq \mathcal{F}$ are the possible preconditions and possible add/delete effects, where $\diamond$ being used to differentiate between the certain and possible predicates.
It implies that a precondition and effect may or may not be present in the original model and are used to model the uncertainty in the robot's understanding of the human's model.
If their are $n$ uncertain preconditions and effects then $|\mathbb{M}| = 2^n$.

A planning problem in the domain is given by $\mathcal{P} = \langle\mathbb{M}, \mathcal{I}, \mathcal{G}\rangle$, where $\mathcal{I}$ is the initial state and $\mathcal{G}$ is the goal. 
$\delta_\mathcal{M}$ is the transition function $\delta_\mathcal{M}: S \times \mathcal{A} \rightarrow S$. It can not be executed in a state $s \not\models pre(a)$; else $\delta_{\mathcal{M}}(s, a) \models s \cup eff^+(a) / eff^-(a)$.
There are two different transition functions for the condition when a specific action in a plan can't be executed (due to $\diamond pre(a)$ being true and not planned for), then the plan either goes to a fail state (pessimistic) or the next action can still be executed (optimistic) in the plan \cite{nguyen2017robust}. An agent assumes an optimistic approach as it is concerned about the interaction and the information about the feasibility of the action.
A plan $\pi$ for the model $\mathcal{M}$ is a sequence of actions $\pi = \langle a_0, a_1, ..., a_n\rangle$. A plan is a valid plan if $\delta_{\mathcal{M}}(\mathcal{I}, \pi) \models \mathcal{G}$. A set of plans $\Pi$ are $\Pi = \{\pi| \forall \pi \delta_{\mathcal{M}}(\mathcal{I}, \pi \models \mathcal{G}\}$


\subsubsection{Motivating example} Figure \ref{fig:scenario} shows the situation where a human teammate working with one robot has to migrate to a new one. We now describe the {\em move} action for the new robot compared to the earlier model. For a robot (such as fetch\footnote{https://fetchrobotics.com/robotics-platforms/fetch-mobile-manipulator/}) it is essential to ensure that the robot's torso is NOT in the stretched state and the arm is tucked close to the body, else it can lead to a fatal accident due to its tendency to topple while turning. The previous model of the robot had no such requirements. The domain for the new robot with uncertainty about the preconditions of a move action, i.e. {\em hand\_tucked} (hand not stretched) and {\em is\_crouch} (torso not stretched). The robot's uncertain model of the human is --
\begin{small}
\begin{verbatim}
(:action tuck
:parameter              ()
:precondition           ()
:possible-precondition  ()
:effect                 (and    (is_crouch)
                                (hand_tucked))
:possible-effect        ()
)
(:action crouch
:parameter              ()
:precondition           ()
:possible-precondition  ()
:effect                 (and    (is_crouch)
:possible-effect        ()
)
(:action move
:parameter              (?from ?to - location)
:precondition           (robot-at ?from)
:possible-precondition  (and    (is_crouch)
                                (hand_tucked))
:effect                 (and    (robot-at ?to)
                                (not (robot-at
                                ?from)))
:possible-effect        ())
\end{verbatim}
\end{small}

There are two unknowns preconditions hence there are four possible models. If the objects are defined in the model are -- {\em roomA} and {\em roomB}. Then a question $Q$, used to elicit human's understanding can be defined as $Q = \langle\mathcal{I}, \mathcal{G}\rangle$, i.e. a planning task,

\begin{lstlisting}
$\mathcal{I}$ = {robot-at(roomA)}, 
$\mathcal{G}$ = {robot-at(roomB)}
\end{lstlisting}
and the possible answers in each model are --
\begin{lstlisting}
$\pi_1$ = $\langle$move(roomA, roomB)$\rangle$, 
$\pi_2$ = $\langle$crouch, move(roomA, roomB)$\rangle$, 
$\pi_3$ = $\langle$tuck, move(roomA, roomB)$\rangle$,
$\pi_4$ = $\langle$crouch, tuck, move(roomA, roomB)$\rangle$
\end{lstlisting}
$\pi_1$ shows that there are no preconditions in the human's model, $\pi_2$ means the precondition of is\_crouch is present, and $\pi_3$ shows that human has hand\_tucked precondition and $\pi_4$ both preconditions are true. Thus, based on the response, the robot can find the specific model of the teammate.

An important point to note is that there might be a correct model for the environment that might be known to the automated agent. However, the agent is only trying to learn the exact human model $\mathcal{M}_H$.
The correct model might be different from the human model and may or may not be part of the set of models $\mathbb{M}$.

Similarly, if we assume that the human has a structured model in the form of STRIPS action or some other causal format, which might be possible in the case of experts. Then the robot can ask if a specific precondition or effect is part of the action in their model. The agent needs to ask these questions due to a lack of knowledge about the human's thought process and understanding of the environment. 

\section*{Problem Formulation}
After describing the robot's understanding of the teammate's model, we are now ready to formally define the questions to elicit the behavior of the teammate.

\begin{definition}
Question is a tuple, $\langle \mathcal{I}, \mathcal{G}\rangle$, i.e. initial and the goal states. The solution to a question are plans $\pi$ such that $\delta(\mathcal{I}, \pi) \models \mathcal G$. A sequence of questions as $Q = \langle\langle \mathcal{I}_0, \mathcal{G}_0\rangle, \langle\mathcal{I}_1, \mathcal{G}_1\rangle, ..., \langle\mathcal{I}_n, \mathcal{G}_n\rangle\rangle$, the tuple of sequential planning tasks where predicates and actions are based on the agent's uncertain model.
\end{definition}
As per the definition, a {\em question} does not take into account the response from the user. In a lot of previous work, responses have been to define the state of failure \cite{verma2020asking}, or provide specific model change and the model as the response \cite{bryce2016maintaining}. However, we believe that giving such responses can be computationally taxing for the teammate. The question was framed using a structured model, but the responses are simple like -- (1) whether the query has a valid plan, and (2) the plan for the query in the teammate's model. Before presenting our analysis for each type of query in the next section, we first define the problem of asking questions from the teammate and then discuss how uncertainty in the model affects these queries and the plan constructed by the teammate.


\subsubsection{Question Framing Problem} Given the tuple $\langle \mathbb{M}, \mathcal{I}_e\rangle$, find the sequence of questions $Q$ to learn the correct model with $|Q| \leq n$, where $n$ is the number of possible predicates.

$\mathbb{M}$ represents the annotated model of the robot and $\mathcal{I}_e$ represents the fully-specified input state of the environment in which we are interacting with the human-in-the-loop. Defining the initial state of the environment is useful, to support the interaction between a teammate and the robot. Initial state $\mathcal{I}_q$ for any query $q$ can be specified as a set difference from the fully-specified initial state of the environment $\mathcal{I}_e$.


\subsubsection{Constrained and Relaxed models.} As we described earlier, for any $n$ possible pre-conditions and effects, hence, there is an exponential set of possible models $(2^n)$. It is computationally infeasible to analyze all of them to generate questions. Thus, we look at two extreme models in the robot's model -- (1) most constrained model $\con$ and (2) most relaxed model $\rel$.
$\con$ is the model where only $\pre(a) = \pre(a) \cup \ppre$, and $\de(a) = \de(a) \cup \pde$. $\rel$ is the model where $\pe = \pe \cup \ppe$.
Let $\Pi_{\con}^q$ represent the set of solutions to any question $q$ in model $\con$ and similarly, $\Pi_{\rel}^q$ represent the If a plan exists in $\rel$, then it exists in the rest of the model and the plan length is minimum in $\min$ and highest in the $\con$ \cite{sreedharan2018handling}. 
The cost of actions is the same across all the concrete models in the robot's model (as one of them represents the model of the same hilp), $C^*_{\con}$ is the cost of the optimal plan in $\con$, and similarly $C^*_{\rel}$ for $\rel$. The corollary follows from the construction of relaxed and constrained models \cite{sreedharan2018handling}.

\begin{coro}
\label{cor:constrained_relaxed_comparison}
$C^*_{\rel} \leq C^*_{\con}$ and $\Pi_{\con}^q \subseteq \Pi_{\rel}^q$.
\end{coro}

Equality in the corollary exists when $\rel$ is the same as $\con$, and there is no uncertainty in the model. 

\subsubsection{Human Model.} We assume that the human model is one of the models in the set of robot's model $\hum \in \mathbb{M}$. Thus when presented with a question $q$, let $\Pi_{\hum}^q$ represent the possible set of solutions in the human's model and $C^*_{\hum}$, is the cost of the optimal plan in the human's model. Given this current setup, we can say a few things about the solution provided by the user to any question.

\begin{coro}
Assuming user to be an optimal planner then for a specific question, $C^*_{\rel} \leq C^*_{\hum} \leq C^*_{\con}$.
\end{coro}

\begin{coro}
If we relax the assumption of the optimal planner then we can say that, $\Pi_{\con}^q \subseteq \Pi^q_{\hum} \subseteq \Pi_{\rel}^q$.
\end{coro}

The corollary explains that the optimal solution constructed by hilp follows these bounds. The corollary follows from corollary \ref{cor:constrained_relaxed_comparison}, where the human model is either the most constrained or relaxed model for a subset of possible predicates. For comparison between $\hum$ and $\rel$, the $\hum$ is the most constrained version of the $\rel$ model for a subset of constraints. Similarly, while comparing $\hum$ and $\con$, we can see that $\hum$ is the most relaxed version of the $\con$ model on the subset of original constraints. 
Showing the bounds for the possible human solution to any question doesn't take the human model into account and applies to any model in $\mathbb{M}$. These models help us analyze different cases to formulate the properties of the queries based on the plans that are possible in these models in the next section.

\section{Distinguishing Query}
Due to the combinatorial explosion, the agent can not search through all possible initial and goal states, and the agent needs to generate a query from the bounds on the human model. Thus, to analyze the specific initial and the goal state, we look at the properties of the plans which should be discussed and then find the initial and goal state that will support their execution. We start our analysis with the simplest case, how to distinguish models with a single probable predicate (i.e., a single unknown predicate in an in the model), and then generate a directed question when there are many possible predicates in the model.

The solution provided by the user depends on the question posed by the agent, and the query depends on the annotated constraint. If $p$ is the predicate possibly  (for now assuming it's a pre-condition) present in action $a$ (we should be writing $\mathcal{M}^{a\diamond p}$, but for simplicity, we write $\mathcal{M}^{\diamond p}$). Thus there can be four different models which the agent needs to consider --

\begin{itemize}
    \item[*] $\con^{-p}$ -- Constrained model with $p$ is not part of it.
    \item[*] $\rel^{+p}$ -- Relaxed model with $p$ is part of the model.
    \item[*] $\hum^{-p}$ -- Human model with $p$ is not part of the model.
    \item[*] $\hum^{+p}$ -- Human model with $p$ is part of the model.
\end{itemize}
Please note, a predicate $p$ being part of the model means that it is present in the specific action $a$, for which it was a possible predicate. It does {\em not} mean an abstracted predicate from the model. Also, note that only one of the models $\hum^{-p}$ and $\hum^{+p}$ is the true model. Given, the solution $\pi_H$ (provided by hilp of cost $C_H$), and assuming human is an optimal planner, there can be a few possibilities --

\begin{enumerate}
    \item $C^*_H < C^*_{\rel^{+p}}$, would mean that $\pi_H \not \in \Pi_{\rel^{+p}}$ as the solution plan cost provided by hilp is less than the optimal plan cost in $\rel^{+p}$. Thus, $\hum^{-p}$ is the real human model, i.e. the constraint is not part of the human model.
    \item $C^*_H \geq C^*_{\con^{-p}}$, would mean that $\pi_H \not \in \Pi_{\con^{-p}}$ as it is not a valid plan in $\con^{-p}$. Thus, $\hum^{+p}$ is the real human model, i.e. the constraint is part of the human model.
\end{enumerate}

In principle, the plan for the question posed by the agent to determine a constraint $p$ will be harder to execute in $\hum^{+p}$ and easier in $\hum^{-p}$ and help us distinguish the models, which can be achieved by tweaking the initial state of the query. For example, if $p$ is not part of the initial state, then the human needs to achieve this pre-condition to execute the action, only if the pre-condition is part of the human model. Similarly, we can extend this idea to any possible effect $e$, where a plan involving the action would be costlier to generate using $\hum^{-e}$, as compared to $\hum^{+e}$ by changing the goal state (in different scenarios, i.e., with or without the effect $e$). We call such queries {\em distinguishing query}.

\begin{definition}
A question $Q_p$ is a {\em distinguishing query} if the plan for it can distinguish between models $\hum^{+p}$ and $\hum^{-p}$.
\end{definition}

Until now, we analyzed different possible models and the set of responses based on the optimal plan in each model. In the next subsection, we look at the properties of the plan that will be provided by the user $\pi_H$ for yes/no response.

\subsection{Properties}
We have outlined the central idea where the action with an unknown predicate has to be part of the plan for a distinguishing query. Now, we will formally describe these properties and use them to generate a query.
For this, we introduce a topic from the planning literature called {\em landmarks} to explain the properties of the plan.

\subsubsection{Landmark.} A {\em landmark} $L$ is a logical formula, where, $\forall \pi: \delta(\mathcal{I}, \pi) \models \mathcal{G}$ and for some prefix of plan $\pi_{pre} = \langle a_0, a_1, ..., a_i\rangle, i < n$, $\delta(\mathcal{I}, \pi_{pre}) \models L$.
An {\em action} landmark is for any action $a$,$\forall \pi:\delta(\mathcal{I}, \pi) \models \mathcal{G}, a \in \pi$ \cite{keyder2010sound}. For example, the $\mathcal{G}$ is a trivial landmark and if there exists only one action $a$ that reaches a particular fact $f \in \mathcal{G}$, then the action is an action landmark. Finding a landmark is a PSPACE-complete problem \cite{hoffmann2004ordered}.

\begin{prop}
\label{prop:simple_distinguishing}
Necessary condition for a question $Q_p$ to distinguish between models $\mathcal{M}^{+p}$ and $\mathcal{M}^{-p}$ is $a_p \in L(\mathcal{M}^{-p})$, where $a_p$ is the action with proposition $p$ is a landmark in model $\mathcal{M}^{-p}$ for the query.
\end{prop}
\begin{proof}
Proof is divided into two parts. First, we show that being a landmark is necessary and then we show that landmark has to be of model $\mathcal{M}^{-p}$. 
For the first part, let's assume if there is a plan $\pi$ for a distinguishing problem $Q_p$ and such that $a_p \not \in \pi$. Then $\delta_{\mathcal{M}^{+p}}(\mathcal{I}_{Q_p}, \pi) \models \mathcal{G}_{Q_p}$ and $\delta_{\mathcal{M}^{-p}}(\mathcal{I}_{Q_p}, \pi) \models \mathcal{G}_{Q_p}$, i.e. the plan can be executed in both the models. Which refutes the assumption that $Q_p$ is a distinguishing problem.

For the second part, observe that $\Pi^{\mathcal{M}^{+p}}_{Q_p} \subseteq \Pi^{\mathcal{M}^{-p}}_{Q_p}$, i.e. every plan possible in more constrained model ($\mathcal{M}^{+p}$) is also a plan in the less constrained model ($\mathcal{M}^{-p}$). 
Again, if we assume that the action $a_p$ is a landmark in $\mathcal{M}^{-p}$, then there are some solutions $\pi \in \Pi^{\mathcal{M}^{-p}}_{Q_p} \setminus \Pi^{\mathcal{M}^{+p}}_{Q_p}$, where $a_p$ is not a landmark. Following the previous corollary, we can conclude that all plans have the property $a_p \in L(\mathcal{M}^{-p})$.
\end{proof}

An important point to note here is that we do not distinguish between fact $p$ being a pre-condition or effect for action because this condition is necessary for either of them. The agent has to ensure that the action with a possible predicate is a landmark action. One of the methods is to use the {\em add effects} as the goal of the question. In theory, we can use this method to ensure that the action $a$ is a landmark action in the human's model, but in practice, we need to assume that the human is an optimal planner, and thus the action is an optimal landmark. Removing this assumption is out of scope for this work and will be part of our future work.
Proposition \ref{prop:simple_distinguishing} can be extended, that by ensuring that action $a$ is a landmark in $\mathcal{M}^{-p}$, it is harder to achieve all pre-condition for the action $a$ in the constrained model $\mathcal{M}^{+p}$. The proposition directly leads to two corollaries that can establish the distinguishing property.
 
\begin{coro}
The distinguishing problem $Q_p$ should have atleast one solution in $\mathcal{M}^{-p}$.
\end{coro}

\begin{coro}
The distinguishing problem $\mathcal{Q}_p$ should not be solvable in $\mathcal{M}^{+p}$.
\end{coro}

Both the above corollaries follow from the way constrained and relaxed models are constructed, where distinguishing problem $Q_p$ should have atleast one solution in $\mathcal{M}^{-p}$, and no solutions in $\mathcal{M}^{+p}$. Thus, if the problem is solvable 
(either by executing the plan in the environment or asking the human for the plan in their model) 
then the fact is fictitious otherwise, the fact is real. The answer to the distinguishing query is a simple yes/no, facilitating the interaction with the human-in-the-loop.


\subsubsection*{Difference between Pre-conditions and Effects.} The analysis stands true for any possible predicate, be it pre-condition or an effect, and results in different ways of constructing models $\mathcal{M}^{-p}$ and $\mathcal{M}^{+p}$. The basis for constructing these models is that $\mathcal{M}^{+p}$ is the constrained model, and $\mathcal{M}^{-p}$ is the relaxed one. Thus for any possible pre-condition and a delete effect -- $\mathcal{M}^{+p}$ is where the predicate is true, and $\mathcal{M}^{-p}$ is where the fact is not part of the action in the model. Conversely, for add effects $\mathcal{M}^{+p}$ -- the predicate is not part of the action in the model, and $\mathcal{M}^{-p}$ -- predicate is part of the action in the model.


\subsection{Proposition Isolation Principle (PIP)}
\subsubsection{Plan Generation Queries} In this section, we discuss how an agent can ask questions to ensure that uncertainty about a predicate is not affected due to interaction with another possible predicate. It is a brute force method to ask questions by isolating the predicate, i.e., there will be $n$ questions for $n$ possible unknowns. 
The PIP method involves constructing the two models --
\begin{itemize}
    \item $\mathcal{M}^{+p}$ is the most constrained model $\con$.
    \item $\mathcal{M}^{-p}$ is $\con^{-p}$, i.e. most constrained model where the predicate is not present in case of pre-conditions and delete effects, and considered part of the model in case of add-effects.
\end{itemize}
It naturally follows that the models differ by a single predicate. Using $\con$ to pursue a landmark essentially helps isolate the specific predicate $p$ by supporting both the presence and absence of other predicates. For example, an action in the plan may have a possible predicate as a pre-condition, then it is added to the initial state of the query.

\begin{prop}
\label{prop:general_distinguishing}
Any question $Q_p$, which distinguishes the model $\mathcal{M}^{+p}$ and $\mathcal{M}^{-p}$, is isolated by $\con$ and $\con^{-p}$.
\end{prop}
\begin{sketch}
To prove the statement, we need to understand that the models differ due to the predicate $p$, thus all the plans will only differ due to the absence (presence, in case of add effect) of the predicate in $\con$.
Since, the plans are being constructed for $\con^{-p}$ the agent ensures that other possible or valid constraints (predicates in the action) are either satisfied through the initial state or are easily achievable by actions in the plan.
In other words, the plans are possible in all the other models, and the failure is due to the constraint (predicate) $p$. This completes the sketch.
\end{sketch}

The proposition \ref{prop:general_distinguishing} explains that by using the PIP method, the agent can ask questions about every predicate. It works on the idea of providing all the pre-conditions that affect the action $a_p$. If the user might fail to execute a plan involving $a_p$, then it is only due to the constraint $p$. 
To use a part of the model as the running example, if we wanted to isolate the predicate of hand\_tucked and for the scenario please assume tuck action is not present in any of the models, then the ``plan'' generation query for the teammate will be --
\begin{lstlisting}
$\mathcal{I}$ = {is_crouch, robot-at(roomA)}, 
$\mathcal{G}$ = {robot-at(roomB)}
\end{lstlisting}
The teammate will either respond that a plan is possible in their model, or it is not possible to reach the goal. The robot can understand each possible response, where ``yes'' means hand\_tucked precondition is not part of the human model and ``no'' means it is a precondition in the humans model. An important point to note is that another possible predicate is\_crouch was made possible in the initial state, thus, ensuring that if the user responds ``no'', it could only be because of the presence of hand\_tucked precondition, thus explaining the PI principle. This query depends on the idea that {\em tuck} action is not part of the teammate's model else this query can not distinguish based on a yes/no response from them. Now we discuss when any question is a distinguishing query.

\subsubsection{Sufficiency Condition.} 
The agent uses a $\con$ model to find the solution with the landmark. However, it doesn't prove that the same or similar plan involving the landmark action would be optimal in the human model.
Due to fewer constraints compared to $\con$, there could be many other actions that could be used by the agent and thus might have a different plan which might not involve the action.
Thus, the agent needs to find this scenario and prevent it by increasing the cost of executing other actions.

To understand these scenarios, the agent can convert action to a landmark by adding it's $eff^+(a_p)$ as the goal. But it is always possible to have two or more actions that provide a subset of the goal predicates (whose union is the goal set). In the teammate's model, the plan with these actions could be smaller. Thus, the agent needs to ensure that these actions are harder to execute in any less constrained model. Another scenario could be when an action that provides the precondition $p$ for action $a_p$ also satisfies another pre-condition in it. Thus, when the teammate's plan includes the action which provides the possible pre-condition, the agent can't be sure the action is for pre-condition $p$ or not. It could be impossible to check all the action combinations for every possible condition, but once the plan for the query is known in $con$, the number of actions in the plan are finite, and thus evaluating these conditions is feasible.

\subsubsection{Solving (Plan) v$\backslash$s Validating (Val query).} 
We have outlined a method that will be discussed in the solution section, but there are still conditions that might ensure the teammate's plan might not include the $a_p$.
An agent can continue to constrain the set of possible plans by removing actions from the model or some other condition. 
However, these extra additions can overwhelm the teammate as the communication keeps getting complex and can burden the teammate in finding a plan. 
The agent can ask the query as the initial, the goal state, and the plan. The teammate can respond to the validity of the plan in the model. In this scenario, the agent needs to ensure that the plan can be executed only in the $\con^{-p}$.
Where solving a problem would have given more information, but currently, our framework handles the human as an optimal agent. Thus, the agent can still ask the plan constructed with the initial and goal state in $\con^{-p}$ for validation. 

Let's continue with the running example and since {\em tuck} action is part of the teammate's model, we can't ask the teammate whether a plan is possible in their model and thus the plan query is not a distinguishing query has a possible plan of {\em tuck} and {\em move} or just {\em move} in the candidate models. Thus, we need to ask it like a validate query where we validate for the constrained model.
\begin{lstlisting}
$\mathcal{I}$ = {is_crouch, robot-at(roomA)}, 
$\mathcal{G}$ = {robot-at(roomB)}
$\pi$ = $\langle$move(roomA, roomB)$\rangle$.
\end{lstlisting}
Now if the user responds ``yes'' hand\_tucked is not part of the teammate.s model, and if they respond ``no'' then hand\_tucked is part of their model.



\section*{Decreasing Questions}
In the previous section, we have shown how the agent can interact with the teammate, and ensure that every interaction can be useful. However, if the agent wants to decrease the number of questions, it has to question more than one predicate, and thus, there could be multiple reasons for the infeasibility of the query.
There are conditions when both $\mathcal{M}^{+p}$ and $\mathcal{M}^{-p}$ have another feasible solution. Thus, using PIP and these conditions, every subset of the models is bound to have an optimal plan and the agent uses the differing plans to infer the teammate's model.
The analysis assumes the teammate is an optimal planner and the positive action cost for the model. We present the step-by-step construction of such questions that we call {\em templates}, as they can be merged to construct a query for more than one predicate, where every subset of the model has a valid and optimal solution.


\subsubsection{Pre-conditions.} For a proposition $p$ which is a possible pre-condition of the action $a_p$. For the action to executed $pre(a_p)$ can be provided by -- (1) initial state, or (2) executing another action $a'$ where $p \in eff(a')$.
When the precondition comes from the initial state, then the distinguishing query can construct using an isolated proposition. Now, we discuss how such a template exists and how a query can ensure different plans in both the models ($\mathcal{M}^{+p}$ and $\mathcal{M}^{-p}$).

\begin{prop}
\label{prop:template_pre}
Distinguishing question for $\diamond p$ of an action $a_p$ has a distinct valid plan in models $\mathcal{M}^{+p}$ and $\mathcal{M}^{-p}$ when for another action $a'$, $p \in eff^+(a')$.
\end{prop}
\begin{proof}
Assume for a given distinguishing problem $Q_p$ actions $a \in L(\mathcal{M}^{-p})$ and $a'$ can be executed i.e. $\mathcal{I}_{Q_p} = \{pre(a) \cup pre(a') \} \setminus \{p\}$ and $\mathcal{G}_{Q_p} = \{eff^+(a)\} \setminus \{p\}$ . For, model $\mathcal{M}^{+p}$ the plan is $\pi = \langle a', a\rangle$. This plan is also a valid plan in $\mathcal{M}^{-p}$. But due to optimality and non-zero action costs, the plan $\pi$ is not an optimal plan in $\mathcal{M}^{-p}$ as pre-condition provided by $a'$ is not required in the model to execute action $a$. Thus, the optimal plan will be $\pi' = \langle a \rangle$ which is distinct.
\end{proof}

\subsubsection{Add Effects.} For possible add effects $p = \diamond eff^c(a)$ an action $a'$ such that $p \in pre(a')$. If there is another action $a''$ where $p \in eff^+(a'')$. Now we will discuss the template in some more detail.

\begin{prop}
\label{prop:template_eff}
Given three actions, $a, a', a''$ where $p = \diamond eff^c(a)$ $\diamond eff^c(a) \in pre(a')$ and $\diamond eff^c(a) \in eff^c(a'')$ will have distinct plans in models $\mathcal{M}^{+p}$ and $\mathcal{M}^{-p}$.
\end{prop}
\begin{proof}
We will again use proof by construction. Consider a distinguishing proble $Q_p$, where $\mathcal{I}_{Q_p} = \{ pre(a) \cup pre(a') \cup  pre(a'') \} \setminus \{p\}$ and $\mathcal{G}_{Q_p} = \{ eff^+(a) \cup eff^+(a')\} \setminus \{p\}$. This will ensure actions $a$ and $a'$ are landmarks and $a''$ can be executed. Now in the case of $\mathcal{M}^{-p}$ (remember constructive effect are part of less constrained model, follows from proposition \ref{prop:general_distinguishing}), $\pi = \langle a, a' \rangle$, and for model $\mathcal{M}^{+p}$, $\pi = \langle a, a'', a' \rangle$, which completes the construction.
\end{proof}

The limitation of asking about binary interaction is that the agent is not using the model structure. However, using the templates checks whether a particular causal dependency exists in the teammate's model. If it is not part of the true model then an optimal plan will not use the specific action in the plan that provides a relationship. In this case, we need to be sure there isn't another causal dependency among those actions, and in that case, the agent can't be certain. The PIP principle holds for asking questions like this and the agent has to isolate the predicates that it wants to use for asking questions. Extra care has to be taken to ensure that the predicates that are used as a template -- (1) do not have more than one causal dependency, and (2) do not have mutually exclusive relationships in the plans.

If we look at the motivating example and assume that the teammate has the {\em tuck} action in their model. Thus a plan generation query is a template query here --

\begin{small}
\begin{lstlisting}
$\mathcal{I}$ = {is_crouch, robot-at(roomA)}, 
$\mathcal{G}$ = {robot-at(roomB)}
\end{lstlisting}
\end{small}
and in this case the possible responses are --
\begin{small}
\begin{lstlisting}
$\pi_{\con}$ = $\langle$tuck, move(roomA, roomB)$\rangle$
$\pi_{\con^{-p}}$ = $\langle$move(roomA, roomB)$\rangle$
\end{lstlisting}
\end{small}
Two important points to note here, are how PIP was applied for isolating hand\_tucked precondition and the answer to the same question (as plan generation query), requires more information to refine the teammate's model.

It is possible to combine the templates because every sub-model (with or without the predicates) has a valid plan and if there aren't any negative interactions between pair-wise actions in the plan. Since its not feasible to check all the cases, we construct a planning problem from the initial state of the environment which will be presented in the next section.

\begin{coro}
Given no destructive interactions between the actions, the templates can be combined where each sub-space, such as $\mathcal{M}^{+p_1,-p_2}$ will have a distinct plan.
\end{coro}
The proof follows from the assumption of lack of destructive interactions as the plan for each sub-space is a union of distinct plans in the template. We can merge the questions and even ask the teammate which plan from the set of plans is valid in their model.
If we check the motivating example, it shows the case of two template queries for questioning about is\_crouch and hand\_tucked predicate in {\em move} action.

\begin{algorithm}[tbp]
\SetAlgoNoLine
\SetArgSty{textnormal}
\SetKwInOut{Input}{Input}
\SetKwInOut{Output}{Output}
\Input{$\mathcal{I}_e, \con^{-p}, {a'}, {a_p}$}
\Output{$\langle \mathcal{I}_q, \mathcal{G}_q, \pi_q\rangle$}
\Begin{
    $\mathcal{G} \leftarrow pre(a') \cup pre(a_p)$\;
    $\pi \leftarrow Solve(\langle\con^{-p}, \mathcal{I}_e, \mathcal{G}\rangle)$\;
    $\pi_q \leftarrow \langle\pi, a_p\rangle$\;
    $\mathcal{I}_q, \mathcal{I}_{temp} \leftarrow$ Project$(\mathcal{I}_e, \pi)$\;
    $\mathcal{G}_q, \mathcal{G}_{temp} \leftarrow eff^{+}(a_p)$ \;
    \For{$a_x \in \{a | a \in \pi_q \& eff^{+}(a) \in \mathcal{G}_q\}$}{
        $f \leftarrow \{f| f \in eff^{+}(a_x) \& f\not \in \mathcal{G}_q\}$\;
        $\mathcal{G}_{temp} = \mathcal{G}_{temp} \cup \neg f$\;
    }
    \For{$a_x \in \{a| a\in \pi_q, p \in eff(a), pre(a_p) \setminus p \in eff(a)\}$}{
        $f \leftarrow \{f|f \in pre(a_p), f \in eff(a_x), f \not= p\}$\;
        $I_{temp} = I_{temp} \cup f$\;
    }
    $\pi \leftarrow Solve(\langle\con^{-p}, \mathcal{I}_{temp}, \mathcal{G}_{temp}\rangle)$\;
    \uIf{$a_p \in \pi$}{
        \Return $\langle \mathcal{I}_{temp}, \mathcal{G}_{temp}, \pi\rangle$\;
    }
    \Else{
        \Return $\langle \mathcal{I}_q, \mathcal{G}_q, \pi_q\rangle$\;
    }
}
\caption{Query Generation Algorithm (QGA)}
\label{alg:query_generation}
\end{algorithm}

\section{Proposed Solution}
As we described in the earlier sections, a query is an initial and goal state in which the action $a_p$ is an optimal landmark in the model $\con^{-p}$. The easiest way to achieve this is to define the goal state as the add-effects of the action $a_p$. Finding an initial state is difficult because we need to ensure that given any possible pre-conditions in effect, the interaction should provide information. Thus, in this section, we describe two algorithms -- (1) to iterate over each unknown predicate and decide whether templates can be combined, and (2) for generating the query.

The parent routine is to iterate over each unknown in an order decided by $\joe$, which assumes that all the unknown predicates are true. This model can't be used for analysis, as it is neither most constrained nor most relaxed, and the plans generated may or may not be part of another model. Then we construct a relaxed planning graph with pair-wise mutexes called graph-plan planning graph \cite{kambhampati1997understanding}. It also supports the construction of templates by checking the conditions explained in propositions \ref{prop:template_pre} and \ref{prop:template_eff} are satisfied and mutexes for plans to merge them.


Algorithm \ref{alg:query_generation}, generated queries for specific $p, a_p$ pairs. First, it solves a planning problem, where the plan is to reach the preconditions of the action $a_p$ in the model $\con^{-p}$. If the query is for the template, then preconditions of other actions is used as well (follows from proposition \ref{prop:template_pre} and \ref{prop:template_eff}). 
The solution of the planning problem and then executing action $a_p$ as the goal is $eff^+(a_p)$. 
The projection function finds the subset of the initial state for constructing the plan, to ensure that other plans are not feasible in the teammate's model.
Then we satisfy sufficiency conditions for the plan using $\joe$. From lines 7-10, we handle every effect in $a_p$ that might be provided by other actions. The negation of the pre-condition from these actions is added to the initial and goal state.
From lines 11-13, the algorithm checks if an action in the plan threatens $p$. The threat is handled by removing the action from the plan and adding its constraints to the initial state.
Finally, it reevaluates whether the updated initial and goal state constructs the plan in the constrained model or not. If the plan still contains the action $a_p$ then the query is to validate the plan $\pi_q$, instead of asking them to generate the plan. This solution follows the complete analysis to ensure a sufficient and minimal response from the teammate.
The agent can ask the query in any order due to PIP and sufficiency conditions because other constraints do not affect the current query for any predicate. However, the sequence was derived from the graph-plan planning graph for $\con$, based on the action closer to the initial state.




    

\section{Empirical Evaluation}
\begin{table}[]
\small
\centering
\begin{tabular}{c|c|c|c|c|c|c}
\toprule
     & $|\diamond p|$ & $|\mathcal{Q}|$ & Val & Plan & Templ & Time \\ \midrule
    \multirow{3}{*}{Blocks} & 4 & 3.7 & 1.9 & 0.7 & 1.1 & 1.79 \\ \cmidrule{2-7}
    & 6 & 5.4 & 3.1 & 0.8 & 1.5 & 5.62 \\ \cmidrule{2-7}
    & 8 & 7.4 & 3.8 & 1.1 & 2.5 & 12.24 \\ \midrule
    \multirow{4}{*}{Rover} & 4 & 3.8 & 2.0 & 0.8 & 1.0 & 2.21 \\ \cmidrule{2-7}
    & 6 & 5.6 & 3.1 & 0.8 & 1.5 & 7.89 \\ \cmidrule{2-7}
    & 8 & 7.3 & 4.0 & 1.5 & 1.8 & 13.24 \\ \cmidrule{2-7}
    & 10 & 9.4 & 4.8 & 2.4 & 2.2 & 29.53 \\ \midrule
    \multirow{4}{*}{Satellite} & 4 & 3.7 & 1.8 & 0.8 & 1.1 & 2.11 \\ \cmidrule{2-7}
    & 6 & 5.5 & 2.9 & 1.2 & 1.4 & 6.48 \\ \cmidrule{2-7}
    & 8 & 7.4 & 3.9 & 1.8 & 1.7 & 13.69 \\ \cmidrule{2-7}
    & 10 & 9.3 & 4.8 & 2.5 & 2.0 & 25.55 \\ \midrule
    \multirow{3}{*}{ZenoTravel} & 4 & 3.7 & 1.8 & 0.9 & 1.0 & 2.05 \\ \cmidrule{2-7}
    & 6 & 5.4 & 3.0 & 1.0 & 1.4 & 5.93 \\ \cmidrule{2-7}
    & 8 & 7.4 & 3.9 & 1.5 & 2.0 & 12.97 \\ \bottomrule
\end{tabular}
\caption{Comparison of different types of Queries}
\label{tab:query_diff}
\end{table}
We have theoretically discussed the process of generating questions. The question generation method uses APDDL parser \cite{nguyen2013synthesizing} based on PDDLPy\footnote{https://pypi.org/project/pddlpy/}, and an optimal planner Fast downward \cite{helmert2006fast} to solve the planning problems. The results reported are from experiments run on a 12 core Intel(R) Xeon(R) CPU with an E5-2643 v3 @3.40GHz processor and a 64G RAM.
The experiments were performed on -- rover, blocksworld, satellite, and zenotravel\footnote{https://github.com/potassco/pddl-instances}. 
The IPC domains were the correct human model, and randomly chosen predicates were assumed as possible predicates. An equal number of predicates were added to the actions based on the parameters for the action. 
Special care was taken, to ensure that the extra predicate did not make the action impossible by adding a mutex to already available pre-conditions. Please note, that the predicates were randomly removed from the lifted domain, and asking a question about any grounded action will localize the human model in the lifted domain.

\subsubsection{Table \ref{tab:query_diff}}shows the evaluation for the different number of unknowns and the time taken (in seconds) to find the questions for the domains. The number of questions generated and the time are averaged over 10 different runs. The decrease in the questions is because some predicates were merged using the templates. Except for one case in Rover (with 8 unknown predicates), where we were able to find two different merging templates (thus total questions became 6), we usually had roughly 1 question decrease in the problems. The average number of queries are presented in the table for each domain. The algorithm needs to solve multiple planning problems, but due to PIP query generation can be executed offline. All the queries were first constructed and then validated with the human model (correct IPC domain). The table shows that roughly half of the queries were through validation that was higher than our expectations.

\begin{table}[]
\centering
\begin{tabular}{c|c|c|c|c|c}
\toprule
    Objects & $|\mathcal{Q}|$ & Val & Plan & Templ & Time \\\midrule
    8 & 5.6 & 3.1 & 0.8 & 1.5 & 7.89 \\ \midrule
    15 & 5.7 & 3.1 & 1.0 & 1.4 & 8.02 \\ \midrule
    22 & 5.6 & 3.0 & 1.3 & 1.3 & 8.45 \\ \midrule
    29 & 5.5 & 3.2 & 1.1 & 1.3 & 9.11 \\ \bottomrule
\midrule
\end{tabular}
\caption{Effect of $\mathcal{I}_e$ on the Queries for Rover domain}
\label{tab:object_comparison}
\end{table}
\subsubsection{Table \ref{tab:object_comparison}}shows varying the initial state condition on question generation. We constructed new initial states by adding objects to the environment. We randomly removed three and added three different predicates in the lifted domain and generated questions using different initial states. The time and number of questions were averaged over ten different random selections. We expected to observe an effect on time due to extra objects for the time taken to solve multiple planning problems. But, since queries were constructed using Graphplan Planning graph, we did not see any change in time, just a very small increase. It shows that the size of the initial state does not affect the queries, whereas the causal structure of the domain does.






\section{Related Work}
Our work of asking directed questions for model localization, and understanding what every response from the teammate could mean has been motivated by the Intelligent Tutoring System community. But the idea of learning models from data points with specific queries or plan traces has been applied in active learning as well as learning models from plan traces (behavior) for the environment.

\subsubsection{Intelligent Tutoring System} as a community is working towards maximizing the learning of the students for procedural knowledge. Their central goal is to provide a teacher to every student, and the biggest challenge for them is to understand the model of the student from the work they do and provide feedback or new questions to them.
The process of generating questions for students has been used in the past \cite{zhang2016machine}, and they have also used structured knowledge bases to generate more meaningful questions for concepts like photosynthesis \cite{zhang2017adaptively}.
The modeling scheme in ITS is shallow where they represent the knowledge of any concept as a hidden variable using HMM \cite{knowledgetracing}. Parameters learning using sequential data of student's interaction for HMM  \cite{grover2018should} or deep neural networks \cite{piech2015deep}.
Based on the learned models, they have also tried dynamic policies to present questions to students using multi-armed bandits \cite{clement2014online}. 
Our work differs from the ITS community as we are learning a detailed human model for collaboration. We can see this as the first step towards having informative interaction with the user to improve collaboration with them.

\subsubsection{Learning planning model using traces.}
There has been some work to learn the planning models using the behavior in the environment using state predicate differences \cite{gil1994learning,stern2017efficient}, weighted max-sat \cite{yang2007learning} and finite state machines \cite{cresswell2009acquisition,cresswell2011generalised}. Author's \cite{zhuo2020discovering} used deep neural networks to learn shallow machine learning model and use it to predict behavior on the test set. 
There have been other structured formulations such as Linear Temporal Logic (LTL) to represent the knowledge and learning the model using behavior trajectories and an oracle to validate the model \cite{camacho2019learning}.
In \cite{bryce2016maintaining}, authors use a questioning strategy to decrease the number of particles to find how the model has changed from the original behavior. The response to the query is in the form of -- model provided by the user, labeled valid plan, or some specific predicate that is part of the model now. 
Recently, there has been some work to infer the model by asking specific queries in the form of the initial state and the plan \cite{verma2020asking}. Their work differs from our condition as they assume detailed responses of up to which step the plan can be executed in the robot's model, which can easily overwhelm the human teammate (due to interrogative nature).

\subsubsection{Active Learning} has an oracle to question classes of specific data points \cite{settles2009active}. The community learns the underlying model with the help of an all-knowing oracle. The difference with the field is that the data of plan traces is not readily available to the agent, such as in scenarios of stream-based selective sampling \cite{cohn1994neural}. Stream-based active learning ideas \cite{dagan1995committee} are useful for the continuous space of probability distribution but can't be used directly in discrete space of questions, where we instead have the generative model for the traces.

As we can see, the idea of questioning the user (or an oracle) for relevant information is not new. In this paper, we have looked at how it is useful for human-robot teaming. The essential part is to understand how to formally define interaction in the form of question and answer and generate useful queries to localize the model, and the application to the novel area comes with its different challenges.



\section{Conclusion and Future Work}
Through this paper, we have shown how to construct queries with uncertainty in the model. The robot expects simpler answers such as -- (1) yes/no response, and (2) if possible, a complete plan to decrease the number of interactions.
It also evaluates different conditions under which each query will have a response that can refine the set of models. The construction using PIP ensures that every question can be asked in any order and these queries can be constructed offline (provided it knows the set of possible models).
The evaluations show that these questions can be constructed for any unknown predicate in the model.

In the future, we want to analyze a set of plans rather than the optimal plan. It becomes a two-stepped process, a base framework which involves creating questions, and getting information from the response of the user even if that response does not involve the action $a_p$ but still reaches the goal. The agent has to optimize using the value of information to generate queries, instead of assuming an optimal response from the teammate. This analysis would be useful in general decision-making scenarios and can lead to an open-ended discussion with automated agents.

\subsubsection{Acknowledgements} Kambhampati’s research is supported in part by ONR grants N00014-16-1-2892, N00014-18-1-2442, N00014-18-1-2840, N00014-19-1-2119, AFOSR grant FA9550-18-1-0067, DARPA SAIL-ON grant W911NF-19-2-0006, NSF grants 1936997 (C-ACCEL), 1844325, and a NASA grant NNX17AD06G.

\balance
\bibliography{ref}
\bibliographystyle{aaai} 

\end{document}